\documentclass{article} 
\usepackage{hyperref}
\usepackage{url}
\usepackage{amsmath,amssymb,amsthm,bm}
\usepackage{times}
\usepackage{graphicx} 
\usepackage{subfigure}
\usepackage{algorithm}
\usepackage{algorithmic}
\usepackage[numbers]{natbib}

\def\a{{\bf a}}
\def\b{{\bf b}}

\def\g{{\bf g}}

\def\p{{\bf p}}

\def\u{{\bf u}}
\def\v{{\bf v}}
\def\w{{\bf w}}
\def\x{{\bf x}}
\def\y{{\bf y}}

\def\I{{\bf I}}

\def\P{{\bf P}}

\def\U{{\bf U}}

\def\0{{\bf 0}}
\def\1{{\bf 1}}
\def\2{{\bf 2}}
\def\3{{\bf 3}}
\def\4{{\bf 4}}
\def\5{{\bf 5}}
\def\6{{\bf 6}}
\def\7{{\bf 7}}
\def\8{{\bf 8}}
\def\9{{\bf 9}}

\def\EB{{\mathbb E}}

\def\RB{{\mathbb R}}

\newtheorem{theorem}{Theorem}
\newtheorem{lemma}{Lemma}

\newtheorem{assumption}{Assumption}

\begin{document}
\title{Fast Asynchronous Parallel Stochastic Gradient Decent}

\author{Shen-Yi Zhao
Wu-Jun Li \\
Department of Computer Science \\
National Key Laboratory for Novel Software Technology \\
Nanjing University, China
}

\maketitle

\begin{abstract}
Stochastic gradient descent~(SGD) and its variants have become more and more popular in machine learning due to their efficiency and effectiveness. To handle large-scale problems, researchers have recently proposed several
parallel SGD methods for multicore systems. However, existing parallel SGD methods cannot achieve satisfactory performance in real applications. In this paper, we propose a fast asynchronous parallel SGD method,
called AsySVRG, by designing an asynchronous strategy to parallelize the recently proposed SGD variant called stochastic variance reduced gradient~(SVRG). Both theoretical and empirical results show that AsySVRG can outperform
existing state-of-the-art parallel SGD methods like Hogwild! in terms of convergence rate and computation cost.
\end{abstract}

\section{Introduction}
Assume we have a set of labeled instances $\left\{(\x_i, y_i)|i=1,\ldots,n \right\}$, where $\x_i \in \RB^p$ is the feature vector for instance $i$, $p$ is the feature size and $y_i \in \left\{1,-1\right\}$ is the class label of $\x_i$. In machine learning, we often need to solve the following regularized empirical risk minimization problem:
\begin{align}\label{problem}
\underset{\w}{\min}~\frac{1}{n} \sum_{i=1}^n f_i(\w),
\end{align}
where $\w$ is the parameter to learn, $f_i(\w)$ is the loss function defined on instance $i$, and often with a regularization term to avoid overfitting. For example, $f_i(\w)$ can be $\log(1+e^{-y_i \x_i^T \w})$ which is
known as the logistic loss, or $\max\left\{0,1-y_i \x_i^T \w\right\}$ which is known as the hinge loss in support vector machine~(SVM). The regularization term can be $\frac{\lambda}{2}\left\|\w\right\|_2^2$, $\lambda \left\|\w\right\|_1$, or some other forms.

Due to their efficiency and effectiveness, stochastic gradient descent~(SGD) and its variants~\cite{DBLP:conf/nips/Xiao09,DBLP:journals/jmlr/DuchiS09,DBLP:conf/nips/RouxSB12,DBLP:conf/icml/Mairal13,DBLP:conf/nips/Johnson013,DBLP:journals/jmlr/Shalev-Shwartz013,DBLP:conf/nips/Nitanda14} have recently attracted much attention to solve machine learning problems like that in~(\ref{problem}). Many works have proved that SGD and its variants can outperform traditional batch learning algorithms such as gradient descent or Newton methods in real applications.

In many real-world problems, the number of instances $n$ is typically very large. In this case, the traditional sequential SGD methods might not be efficient enough to find the optimal solution for (\ref{problem}). On the other hand, clusters and multicore systems have become popular in recent years. Hence, to handle large-scale problems, researchers have recently proposed several distributed SGD methods for clusters and parallel SGD methods for multicore systems. Although distributed SGD methods for clusters like those in~\cite{DBLP:conf/nips/ZinkevichSL09,DBLP:conf/nips/DuchiAW10,DBLP:conf/nips/ZinkevichWSL10} are meaningful to handle very large-scale problems, there also exist a lot of problems which can be solved by a single machine with multiple cores. Furthermore, even in distributed settings with clusters, each machine~(node) of the cluster typically have multiple cores. Hence, how to design effective parallel SGD methods for multicore systems has become a key issue to solve large-scale learning problems like that in (\ref{problem}).

There have appeared some parallel SGD methods for multicore systems. The round-robin scheme proposed in~\cite{DBLP:conf/nips/ZinkevichSL09} tries to order the processors and then each processor update the variables in order. Hogwild!~\cite{recht2011hogwild} is a lock-free approach for parallel SGD. Experimental results in~\cite{recht2011hogwild} have shown that Hogwild! can outperform the round-robin scheme in~\cite{DBLP:conf/nips/ZinkevichSL09}. However, Hogwild! can only achieve a sub-linear convergence rate. Hence, Hogwild! is not efficient~(fast) enough to achieve satisfactory performance.

In this paper, we propose a fast asynchronous parallel SGD method, called AsySVRG, by designing an asynchronous strategy to parallelize the recently proposed SGD variant called stochastic variance reduced gradient~(SVRG)~\cite{DBLP:conf/nips/Johnson013}. The contributions of AsySVRG can be outlined as follows:
\begin{itemize}
\item Two asynchronous schemes, consistent reading and inconsistent reading, are proposed to coordinate different threads. Theoretical analysis is provided to show that both schemes have linear convergence rate, which is faster than that of Hogwild!
\item The implementation of AsySVRG is simple.
\item Empirical results on real datasets show that AsySVRG can outperform Hogwild! in terms of computation cost.
\end{itemize}


\section{Preliminary}

We use $f(\w)$ to denote the objective function in~(\ref{problem}), which means $f(\w) = \frac{1}{n} \sum_{i=1}^n f_i(\w)$. In this paper, we use $\left\| \cdot \right\|$ to denote the $L_2$-norm $\left\| \cdot \right\|_2$
and $\w_*$ to denote the optimal solution of the objective function.

\begin{assumption}\label{L-smooth}
The function $f_i(\cdot)~(i=1,\ldots,n)$ in~(\ref{problem}) is convex and $L$-smooth, which means $\exists L>0$, $\forall \a,\b, $
\begin{align}
f_i(\a) \leq f_i(\b) + \nabla f_i(\b)^T(\a - \b) + \frac{L}{2} \left\| \a-\b \right\|^2, \nonumber
\end{align}
or equivalently
\begin{align*}
 \left\| \nabla f_i(\a)-\nabla f_i(\b) \right\| \leq L \left\| \a-\b \right\|,
\end{align*}
where $\nabla f_i(\cdot)$ denotes the gradient of $f_i(\cdot)$.
\end{assumption}

\begin{assumption}\label{strong}
The objective function $f(\cdot)$ is $\mu$-strongly convex, which means $\exists \mu>0$, $\forall \a, \b,$
\begin{align}
f(\a) \geq f(\b) + \nabla f(\b)^T(\a - \b) + \frac{\mu}{2} \left\| \a-\b \right\|^2, \nonumber
\end{align}
or equivalently
\begin{align*}
 \left\| \nabla f(\a)-\nabla f(\b) \right\| \geq \mu \left\| \a-\b \right\|.
\end{align*}
\end{assumption}

\section{Algorithm}
Assume that we have $p$ processors~(threads) which can access a shared memory, and $\w$ is stored in the shared memory. Furthermore, we assume each thread has access to a shared data structure for the vector $\w$ and has access to choose any instance randomly to compute the gradient $\nabla f_i(\w)$. We also assume consistent reading of $\w$, which
means that all the elements of $\w$ in the shared memory have the same ``age"~(time clock).

Our AsySVRG algorithm is presented in Algorithm~\ref{alg:AsySVRG}. We can find that in the $t^{th}$ iteration, each thread completes the following operations:
\begin{itemize}
\item All threads parallelly compute the full gradient $\nabla f(\w_t) = \frac{1}{n} \sum_{i=1}^n \nabla f_i(\w_t)$. Assume the gradients computed by thread $a$ are denoted by $\phi_a$ which is a subset of $\{\nabla f_i(\w_t)|i=1,\ldots,n\}$. We have $\phi_a \bigcap \phi_b = empty$ if $a\neq b$, and  $\bigcup_{a=1}^p \phi_a=\{\nabla f_i(\w_t)|i=1,\ldots,n\}$.
\item Run an inner-loop in which each iteration randomly chooses an instance indexed by $i_m$ and computes the gradient $\nabla f_{i_m}(\u_{k(m)})$, where $\u_0 = \w_t$, and compute the vector
\begin{align}\label{def:v}
\v_m = \nabla f_{i_m}(\u_{k(m)}) - \nabla f_{i_m}(\u_0)+\nabla f(\u_0).
\end{align}
Then update the vector
\begin{align*}
\u_{m+1} = \u_{m} - \eta \v_{m},
\end{align*}
where $\eta >0$ is a step size.
\end{itemize}

Here, $m$ is the total number of updates on $\w$ from all threads and $k(m)$ is the $\u$-iteration at which the update $\v_m$ was calculated. Since each thread can compute an update and change the $\w$, $k(m) \leq m$
obviously. At the same time, we should guarantee that the update is not too old. Hence, we need $m - k(m) \leq \tau$, where $\tau$ is a positive integer, and usually called the bounded delay. If $\tau = 0$, the algorithm AsySVRG degenerates to the sequential~(single-thread) version of SVRG.

\begin{algorithm}[!htb]
\caption{AsySVRG}
\label{alg:AsySVRG}
\begin{algorithmic}
\STATE Initialization: $p$ threads, initialize $\w_0, \eta$;
\FOR{$t=1,2,...$}
\STATE All threads parallelly compute the full gradient $\nabla f(\w_t) = \frac{1}{n} \sum_{i=1}^n \nabla f_i(\w_t)$;
\STATE $\u_0 = \w_t$;
\STATE For each thread, do:
\FOR{$m=0$ to $M$}
\STATE Pick up an $i_m$ randomly from $\left\{ 1,\ldots,n \right\}$;
\STATE Compute update vector $\v_m = \nabla f_{i_m}(\u_{k(m)}) - \nabla f_{i_m}(\u_0) + \nabla f(\u_0)$;
\STATE $\u_{m+1} = \u_m - \eta \v_m$;
\ENDFOR
\STATE \textbf{Option 1}: Take $\w_{t+1}$ to be the current $\u$ in the shared memory;
\STATE \textbf{Option 2}: Take $\w_{t+1}$ to be the average sum of $\u_m$ generated by the inner loop;
\ENDFOR
\end{algorithmic}
\end{algorithm}

\section{Convergence Analysis}
Our convergence analysis is based on the Option 2 in Algorithm~\ref{alg:AsySVRG}. Please note that we have $p$ threads and let each thread calculate $M$ times of update. Hence, the total times of updates on $\w$ in the shared memory, which is denoted by $\tilde{M}$, must satisfy that $0<\tilde{M}\leq pM$. And obviously, the larger the $M$ is, the larger the $\tilde{M}$ will be.

\subsection{Consistent Reading}
\label{lock}

Since $\w$ is a vector with several elements, it is typically impossible to complete updating $\w_{m+1}$ in an atomic operation. We have to give this step a lock for each thread.  More specifically, it need a lock whenever a thread tries to read $\u$ or update $\u$ in the shared memory. This is called consistent reading scheme.

First, we give some notations as follows:
\begin{align}
\p_{m,i} &= \nabla f_i(\u_m) - \nabla f_i(\u_0) + \nabla f(\u_0),\label{def:pmi} \\
q_m &= \frac{1}{n} \sum_{i=1}^n \left\| \p_{m,i} \right\|^2. \label{def:qm}
\end{align}

It is easy to find that $\v_m = \p_{k(m),i_m}$ and the update of $\w$ can be written as follows:
\begin{align}\label{update rule}
\u_{m+1} = \u_m - \eta \v_m.
\end{align}

One key to get the convergence rate is the estimation of the variance of $\v_m$. We use the technique in \cite{liu2013asynchronous} and get the following result:
\begin{lemma}\label{lemma1}
There exists a constant $\rho > 1$ such that $\EB q_m \leq \rho \EB q_{m+1}$.
\end{lemma}
\begin{proof}
\begin{align}\label{var:p}
\left\| \p_{m,i} \right\|^2 - \left\| \p_{m+1,i} \right\|^2 &=(\p_{m,i} - \p_{m+1,i})^T(\p_{m,i} + \p_{m+1,i}) \nonumber \\
&=(\p_{m,i} - \p_{m+1,i})^T(2\p_{m,i} + \p_{m+1,i} - \p_{m,i}) \nonumber \\
&\leq 2\p_{m,i}^T(\p_{m,i} - \p_{m+1,i}) \nonumber \\
&\leq 2\left\| \p_{m,i} \right\| \left\| \p_{m,i} - \p_{m+1,i} \right\| \nonumber \\
&= 2\left\| \p_{m,i} \right\| \left\| \nabla f_i(\u_m) - \nabla f_i(\u_{m+1}) \right\| \nonumber \\
&\leq \frac{1}{r}\left\| \p_{m,i} \right\|^2 + r\left\| \nabla f_i(\u_m) - \nabla f_i(\u_{m+1}) \right\|^2 \nonumber \\
&\leq \frac{1}{r}\left\| \p_{m,i} \right\|^2 + rL^2\left\| \u_m - \u_{m+1} \right\|^2 \nonumber \\
&= \frac{1}{r}\left\| \p_{m,i} \right\|^2 + r\eta^2 L^2\left\| \v_m \right\|^2.
\end{align}
The fourth inequality uses Assumption \ref{L-smooth} and $r > 0$ is a constant.
Summing (\ref{var:p}) from $i=1$ to $n$, and taking expectation about $i_m$, we have
\begin{align}\label{eq:p_k}
\EB(q_m - q_{m+1}) \leq \frac{1}{r}\EB q_m + r\eta^2 L^2\EB q_{k(m)}.
\end{align}
We use $c = 2\max \left\{\frac{1}{r}, r\eta^2 L^2\right\}$ and choose $r,\eta$ such that $0<c<1$, then we can get $q_0 \leq \frac{1}{1-c}\EB q_1$. Please note that $k(0) = 0$. Then, we obtain that
\begin{align}
\EB q_m \leq \rho \EB q_{m+1}
\end{align}
where $\rho$ satisfies $\frac{1}{1-c}< \rho$ and $\rho(1-\frac{1}{2}c(1+\rho^\tau)) \geq 1$.
\end{proof}
According to Lemma \ref{lemma1}, $\rho > 1$. If we want $\rho$ to be small enough, we need a small step size $\eta$. This is reasonable because $\u$ should be changed slowly if the gradient applied to update $\u$ is relatively old.

\begin{theorem}\label{theorem1}
With the Assumption \ref{L-smooth} and \ref{strong}, choosing a small step size $\eta$ and large $\tilde{M}$, we have the following result:
\begin{align}
\EB(f(\w_{t+1}) - f(\w_*)) \leq \alpha \EB(f(\w_t) - f(\w_*)) \nonumber
\end{align}
where $\alpha = (\frac{1}{\mu \tilde{M} \eta (1 - 2 (\tau + 1) \rho^{2\tau} \eta L)} +  \frac{2(\tau + 1) \rho^{2\tau} \eta L}{1 - 2 (\tau + 1) \rho^{2\tau} \eta L} )<1$ and $1 - 2 (\tau + 1) \rho^{2\tau} \eta L>0$.
\end{theorem}

\subsection{Inconsistent Reading}
The consistent reading scheme would cost much waiting time because we need a lock whenever a thread tries to read $\u$ or update $\u$. In this subsection, we will introduce an inconsistent reading scheme, in which a thread does not need a lock when reading current $\u$ in the memory. For the update step, the thread still need a lock. Please note that our inconsistent reading scheme is different from that in~\cite{hsieh2015passcode} which adopts the atomic update strategy. Since the update vector applied to $\u$ is usually dense, the atomic update strategy used in~\cite{hsieh2015passcode} is not applicable for our case.

For convenience, we use $\U = \left\{\u_0, \u_1,\ldots\right\}$ to denote the vector set generated in the inner loop of our algorithm, and $\hat{\u}_m$ to denote the vector that one thread gets from the shared memory and uses
to update $\u_m$. Then, we have
\begin{align}\label{update:inconsistent}
&\hat{\v}_m = \nabla f_{i_m}(\hat{\u}_m) - \nabla f_{i_m}(\u_0) + \nabla f(\u_0) \nonumber \\
&\u_{m+1} = \u_{m} - \eta \hat{\v}_m
\end{align}

We also need the following assumption:
\begin{assumption}\label{assum}
For all threads, they enjoy the same speed of reading operation and the same speed of updating operation. And any reading operation is faster than updating operation, which means that for three scalars $a,b,c$, $``b = a"$ is
faster than $``a = a + c"$.
\end{assumption}

Since we do not use locks when a thread reads $\u$ in the shared memory, some elements in $\u$ which have not been read by one thread may be changed by other threads. Usually, $\hat{\u}_m \notin \U$. If we call the age of each element of
$\u_m \in \U$ to be $m$, the ages of elements of $\hat{\u}_m$ may not be  the same. We use $a(m)$ to denote the smallest age of the elements of $\hat{\u}_m$. Of course, we expect that $a(m)$ is not too small. Given a positive integer
$\tau$, we assume that $m - a(m)\leq \tau$. With Assumption \ref{assum}, according to the definition of $\hat{\u}_m$ and $a(m)$, we have
\begin{align}\label{def:hatw}
\hat{\u}_m = \P_{\g_{m,1}}\u_{a(m)} + \P_{\g_{m,2}} \u_{a(m)+1}, 
\end{align}
where $\g_{m,i}$ is a set that belongs to $\left\{1, 2, \ldots,p\right\} (i=1,2)$, $\P_{\g_{m,i}} \in \RB^{p\times p}$ is a diagonal matrix that only the $k^{th}$ diagonal position is $1$, $\forall k \in \g_{m,i}$, and other
elements of $\P_{\g_{m,i}}$ are $0$.

The (\ref{def:hatw}) is right because with an update lock and Assumption \ref{assum}, at most one thread is updating $\u$ at any time. If a thread begins to read $\u$, only two cases would happen. One is that no threads are
updating $\u$, which leads $\P_{\g_{m,2}} = \0$. Another is that one thread is updating $\u$, which leads to the result that the thread would get a new $\u$ and may also get some old elements. Obviously, they enjoy the same age of $\u$ if it
reads at a good pace.

Then, we can get the following results:
\begin{itemize}
\item $a(m) \nleqslant a(m+1)$ or $a(m+1) \nleqslant a(m)$.
\item $\g_{m,1}\cap\g_{m,2} = \emptyset$, and $\g_{m,1}\bigcup\g_{m,2} = \left\{1,2,\ldots,p\right\}$, which means that $\P_{\g_{m,1}} + \P_{\g_{m,2}} = \I_p$, $I_p$ is an identity matrix.
\end{itemize}

Similar to (\ref{def:pmi}) and (\ref{def:qm}), we give the following definitions:
\begin{align}
\hat{\p}_i(\x) &= \nabla f_i(\x) - \nabla f_i(\u_0) + \nabla f(\u_0), \nonumber \\
\hat{q}(\x) &= \frac{1}{n} \sum_{i=1}^n \left\| \hat{\p}_i(\x) \right\|^2 = \EB_i \left\| \hat{\p}_i(\x) \right\|^2, \nonumber \\
\hat{q}_m &= \frac{1}{n} \sum_{i=1}^n \left\| \hat{\p}_i(\hat{\u}_m) \right\|^2. \nonumber
\end{align}
We can find that $\hat{\v}_m = \hat{\p}_i(\hat{\u}_m)$ and $\EB\left\| \hat{\v}_m \right\|^2 = \EB(\EB_{i_m}\left\| \hat{\v}_m \right\|^2) = \EB\hat{q}_m$.

We give a notation for any two integers $m,n$, i.e., $\sum_{i=n}^m = \sum_{i=m}^n = \sum_{i = \min \left\{m,n\right\}}^{\max \left\{m,n\right\}}$.

According to the proof in Lemma \ref{lemma1}, we can get the property that $\forall \x, \y, r>0$,
\begin{align}\label{based ineq}
\left\| \hat{\p}_i(\x) \right\|^2 - \left\| \hat{\p}_i(\y) \right\|^2 \leq \frac{1}{r} \left\| \hat{\p}_i(\x) \right\|^2 + rL^2\left\| \x - \y \right\|^2.
\end{align}

\begin{lemma}\label{estimate:hatq}
There exists a constant $\rho>1$ and a corresponding suitable step size $\eta$ that make:
\begin{align}
\EB\hat{q}_m < \rho \EB\hat{q}_{m+1}. \nonumber
\end{align}
\end{lemma}
\begin{proof} According to (\ref{based ineq}), we have
\begin{align}\label{var:hatp}
\left\| \hat{\p}_i(\hat{\u}_m) \right\|^2 - \left\| \hat{\p}_i(\hat{\u}_{m+1}) \right\|^2 \leq \frac{1}{r}\left\| \hat{\p}_i(\hat{\u}_m) \right\|^2 + rL^2\left\| \hat{\u}_m - \hat{\u}_{m+1} \right\|^2.
\end{align}
According to (\ref{def:hatw}), we have
\begin{align}
     & \left\| \hat{\u}_m - \hat{\u}_{m+1} \right\|^2 \nonumber \\
    =& \left\| \u_{a(m)} - \u_{a(m+1)} + \eta \P_{\g_{m,2}} \hat{\v}_{a(m)} - \eta \P_{\g_{m+1,2}} \hat{\v}_{a(m+1)} \right\|^2 \nonumber \\
\leq & 2\sum_{l=a(m)}^{a(m+1)-1} \left\| \u_l - \u_{l+1} \right\|^2 + 2\eta^2 \left\| \hat{\v}_{a(m)} \right\|^2 + 2\eta^2 \left\| \hat{\v}_{a(m+1)} \right\|^2 \nonumber \\
\leq & 4\eta^2\sum_{l=a(m)}^{a(m+1)} \left\| \hat{\v}_l \right\|^2. \nonumber
\end{align}
In the first inequality, $a(m+1)$ may be less than $a(m)$, but it won't impact the result of the second inequality.
Summing from $i=1$ to $n$ for (\ref{var:hatp}), we can get
\begin{align}
\hat{q}_m - \hat{q}_{m+1} \leq \frac{1}{r}\hat{q}_m + 4r\eta^2 L^2 \sum_{l=a(m)}^{a(m+1)} \left\| \hat{\v}_l \right\|^2 .\nonumber
\end{align}
Taking expectation to $i_{a(m)},i_{a(m)+1},\ldots,i_{a(m+1)}$ which are the random numbers selected by Algorithm \ref{alg:AsySVRG}, we obtain
\begin{align}
\EB\hat{q}_m - \EB\hat{q}_{m+1} \leq \frac{1}{r}\EB\hat{q}_m + 4r\eta^2 L^2 \sum_{l=a(m)}^{a(m+1)} \EB\hat{q}_l . \nonumber
\end{align}
When $\rho, r, \eta$ satisfy the following condition:
\begin{align}
&\frac{1 + 4r\eta^2L}{1-\frac{1}{r} - 4r\eta^2L^2} \leq \rho \nonumber \\
&\rho (1-\frac{1}{r}-4r\eta^2L^2(\tau+1)\rho^\tau)>1+4r\eta^2L^2 ,\nonumber
\end{align}
we have
\begin{align}
\EB\hat{q}_m < \rho \EB\hat{q}_{m+1} \nonumber
\end{align}
\end{proof}

\begin{lemma}\label{estimate:hatvm}
For the relation between $\hat{q}_m$ and $\hat{q}(\u_m)$, we have the following result
\begin{align}
\EB\hat{q}_m \leq c_1 \EB\hat{q}(\u_m). \nonumber
\end{align}
where $c_1 = \frac{1}{1-\frac{1}{r} - 4r\tau\rho^\tau\eta^2L^2}>1$.
\end{lemma}
\begin{proof}
In (\ref{based ineq}), if we take $\x = \hat{\u}_m$ and $\y = \u_m$, we can obtain
\begin{align}
\left\| \hat{\p}_i(\hat{\u}_m) \right\|^2 - \left\| \hat{\p}_i(\u_m) \right\|^2 \leq \frac{1}{r} \left\| \hat{\p}_i(\hat{\u}_m) \right\|^2 + rL^2\left\| \hat{\u}_m - \u_m \right\|^2 .\nonumber
\end{align}
Summing from $i=1$ to $n$, we obtain
\begin{align}\label{EBhatvm}
\hat{q}_m - \hat{q}(\u_m) \leq & \frac{1}{r} \hat{q}_m + rL^2 \left\| \hat{\u}_m - \u_m \right\|^2 \\
                             = & \frac{1}{r}\hat{q}_m + rL^2 \left\| \P_{\g_{m,1}}\u_{a(m)} + \P_{\g_{m,2}} \u_{a(m)+1} - \u_m \right\|^2 \nonumber \\
                          \leq & \frac{1}{r} \hat{q}_m + 4rL^2 \sum_{l=a(m)}^{m-1} \left\| \u_l - \u_{l+1} \right\|^2 \nonumber \\
                             = & \frac{1}{r} \hat{q}_m + 4r\eta^2L^2 \sum_{l=a(m)}^{m-1} \left\| \hat{\v}_l \right\|^2 ,\nonumber
\end{align}
where the second inequality uses the fact that $\P_{\g_{m,1}} + \P_{\g_{m,2}} = \I_p$.

Taking expectation on both sides, and using Lemma \ref{estimate:hatq}, we get
\begin{align}
\EB\hat{q}_m - \EB \hat{q}(\u_m)\leq & (\frac{1}{r} + 4r\tau\rho^\tau\eta^2L^2) \EB \left\| \hat{\v}_m \right\|^2 ,\nonumber
\end{align}
which means that
\begin{align}
\EB\hat{q}_m \leq \frac{1}{1 - \frac{1}{r} - 4r\tau\rho^\tau\eta^2L^2}\EB \hat{q}(\u_m). \nonumber
\end{align}
\end{proof}

Similar to Theorem (\ref{theorem1}), we have the following result about inconsistent reading:
\begin{theorem}\label{theorem:inconsistent}
With Assumption \ref{L-smooth},~\ref{strong},~\ref{assum}, a suitable step size $\eta$ which satisfies the condition in Lemma \ref{estimate:hatq},~\ref{estimate:hatvm}, and a large $\tilde{M}$, we can get our convergence
result for the inconsistent reading scheme:
\begin{align}
\EB(f(\w_{t+1}) - f(\w_*)) \leq (\frac{2}{\mu \tilde{M}(2\eta - c_2)} + \frac{c_2}{2\eta - c_2}) \EB(f(\w_t) - f(\w_*)), \nonumber
\end{align}
where $c_2 = \frac{4L\eta^2 + 16\tau\rho^\tau L^2 \eta^3}{1 - \frac{1}{r} - 4r\tau\rho^\tau\eta^2L^2} < 2\eta$.
\end{theorem}

\textbf{Remark}: In our convergence analysis for both consistent reading and inconsistent reading schemes, there are a lot of parameters, such as $r, \eta, \rho, \tau, \tilde{M}, \mu, L$. We can set $r = \frac{1}{\eta}$. Since $\mu, L$ are determined
by the objective function and $\rho, \tau$ are constants, we only need the step size $\eta$ to be small enough and $\tilde{M}$ to be large enough. Then, all the conditions in these lemmas and theorems will be satisfied.

\section{Experiments}
We choose logistic regression with a $L_2$-norm regularization term to evaluate our AsySVRG. Hence, the $f(\w)$ is defined as follows:
\begin{align}
f(\w) = \sum_{i=1}^n \log(1+e^{-y_i \x_i^T \w}) + \frac{\lambda}{2}\left\| \w \right\|^2. \nonumber
\end{align}

We choose  Hogwild! as baseline because Hogwild! has been proved to be the state-of-the-art parallel SGD methods for multicore systems~\cite{recht2011hogwild}. The experiments are conducted on a server with 12 Intel cores and 64G memory.

\subsection{Dataset and Evaluation Metric}
We choose three datasets for evaluation. They are \emph{rcv1, real-sim}, and \emph{news20}, which can be downloaded from the LibSVM website~\footnote{http://www.csie.ntu.edu.tw/~cjlin/libsvmtools/datasets/}. Detailed information is shown in Table~\ref{data}, where $\lambda$ is the hyper-parameter in $f(\w)$.
\begin{table}[htb]
  \caption{Dataset}\label{data}
  \centering
  \begin{tabular}{|c|c|c|c|}
     \hline
     dataset & instances & features & $\lambda$ \\ \hline
     rcv1 & 20,242 & 47,236 & 0.0001 \\ \hline
     real-sim & 72,309 & 20,958 & 0.0001 \\ \hline
     news20 & 19,996 & 1,355,191 & 0.0001 \\
     \hline
   \end{tabular}
\end{table}

We adopt the speedup and convergence rate for evaluation. The definition of speedup is as follows:
$$
speedup = \frac{CPU~time~taken~to~get~a~suboptimal~solution~with~one~thread}{CPU~time~taken~to~get~a~suboptimal~solution~with~p~threads}
$$
We get a suboptimal solution by stopping the algorithms when the gap between training loss and the optimal solution $\min\left\{f(\w)\right\}$ is less than $10^{-4}$.

We set $M$ in Algorithm \ref{alg:AsySVRG} to be $\frac{2n}{p}$, where $n$ is the number of training instances
and $p$ is number of threads. When $p=1$, the setting about $M$ is the same as that in SVRG~\cite{DBLP:conf/nips/Johnson013}. According to our theorems, the step size should be small. However, we can also get good
performance with a relatively large step size in practice. For the Hogwild!, in each epoch, we run each thread $\frac{n}{p}$ iterations. We use a constant step size $\gamma$, and we set $\gamma \leftarrow 0.9\gamma$ after every epoch. These settings are the same as those in
the experiments in Hogwild!\cite{recht2011hogwild}. For each epoch, our algorithm will visit the whole dataset three times and the Hogwild! will visit the whole dataset only once. To make a fair comparison about the convergence rate, we study the change of objective value versus the number of \emph{effective passes}. One effective pass of the dataset means the whole dataset is visited once.

\subsection{Results}
In practice, we find that our AsySVRG algorithm without any lock strategy, denoted by \mbox{AsySVRG-unlock}, can achieve the best performance. Table~\ref{exp1} shows the running time and speedup results of consistent reading, inconsistent reading, and unclock schemes for AsySVRG on dateset rcv1. Here, 77.15s denotes 77.15 seconds, 1.94x means the speedup is 1.94, i.e., it is 1.94 times faster than the sequential~(one-thread) algorithm.
\begin{table}[htb]
  \caption{Lock versus Unlock~(in second)}\label{exp1}
  \centering
  \begin{tabular}{|c|c|c|c|}
     \hline
     threads & consistent reading & inconsistent reading & AsySVRG-unlock \\ \hline
      2 & 77.15s/1.94x & 77.20s/1.94x & 137.55s/1.09x \\ \hline
      4 & 62.20s/2.4x & 51.06s/2.93x & 58.07s/2.58x \\ \hline
      8 & 63.05s/2.4x & 53.93s/2.78x & 30.49s/4.92x \\ \hline
      10 & 64.76s/2.3x & 56.29s/2.66x & 26s/5.77x \\
     \hline
   \end{tabular}
\end{table}

We find that the consistent reading scheme has the worst performance. Hence, in the following experiments, we only report the results of inconsistent reading scheme, denoted by AsySVRG-lock, and AsySVRG-unlock.

Table~\ref{time} compares the time cost between AsySVRG and Hogwild! to achieve a gap less than $10^{-4}$ with 10 threads. We can find that our AsySVRG is much faster than Hogwild!, either with lock or without lock.
\begin{table}[htb]
  \caption{Time~(in second) taken by 10 threads when the gap is less than $10^{-4}$}\label{time}
  \centering
  \begin{tabular}{|c|c|c|c|c|}
     \hline
     ~ & AsySVRG-lock & AsySVRG-unlock & Hogwild!-lock & Hogwild!-unlock \\ \hline
     rcv1 & 55.77 & 25.33 & $>$500 & $>$200 \\ \hline
     real-sim & 42.20 & 21.16 & $>$400 & $>$200 \\ \hline
     news20 & 909.93 & 514.50 & $>$4000 & $>$2000 \\
     \hline
   \end{tabular}
\end{table}

Figure~\ref{fig} shows the speedup and convergence rate on three datasets. Here, AsySVRG-lock-10 denotes AsySVRG with lock strategy on 10 threads. Similar nations are used for other settings of AsySVRG and Hogwild!. We can find that the speedup of AsySVRG and Hogwild! is comparable. Combined with the results in Table~\ref{time}, we can find that Hogwild! is slower than AsySVRG with different numbers of threads. From Figure~\ref{fig}, we can also find that the convergence rate of AsySVRG is much faster than that of Hogwild!.

\begin{figure*}[htb]
\begin{center}
\subfigure[rcv1]{\includegraphics[width=1.5in]{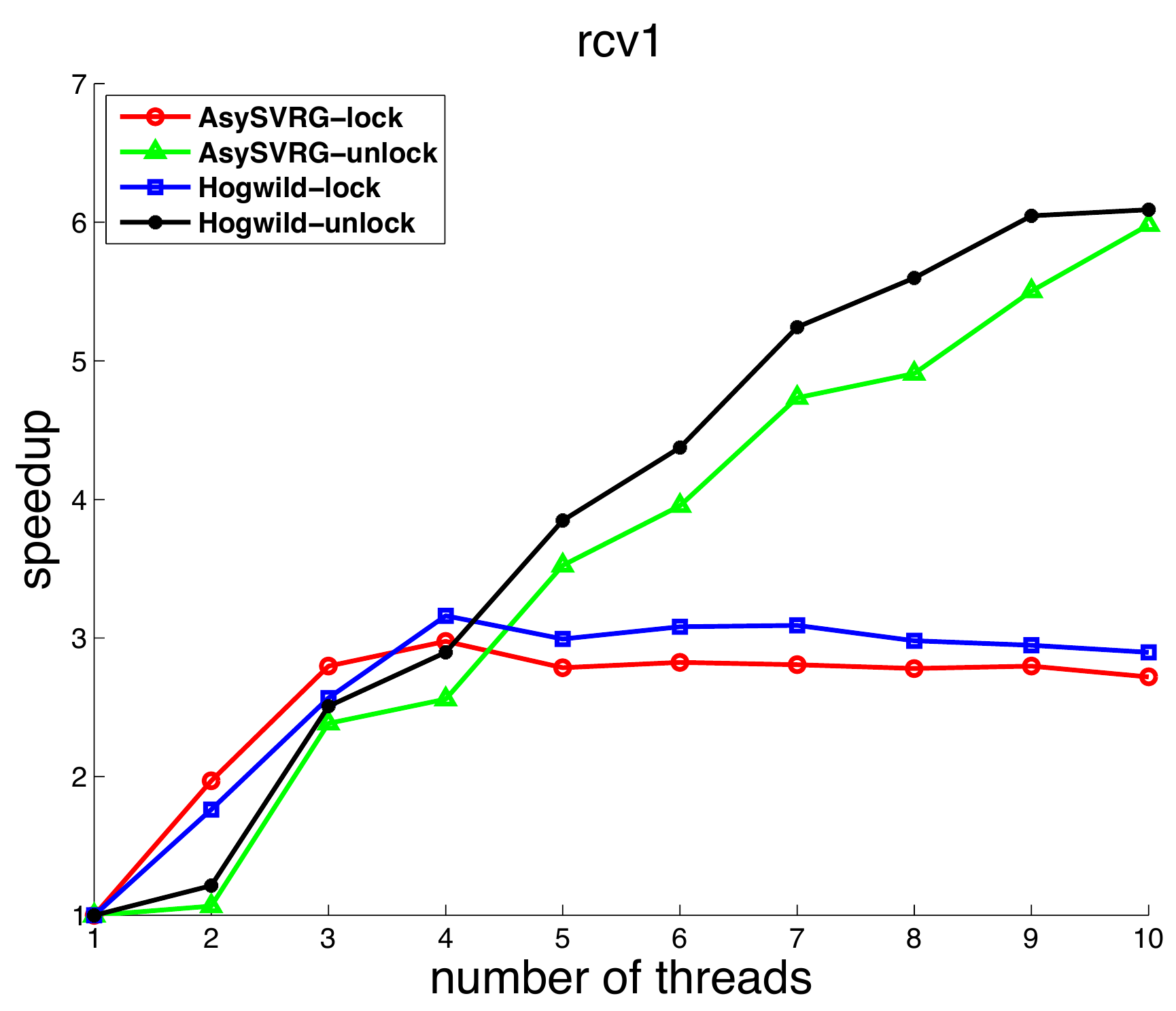}}\hspace{0.5cm}
\subfigure[rcv1]{\includegraphics[width=1.5in]{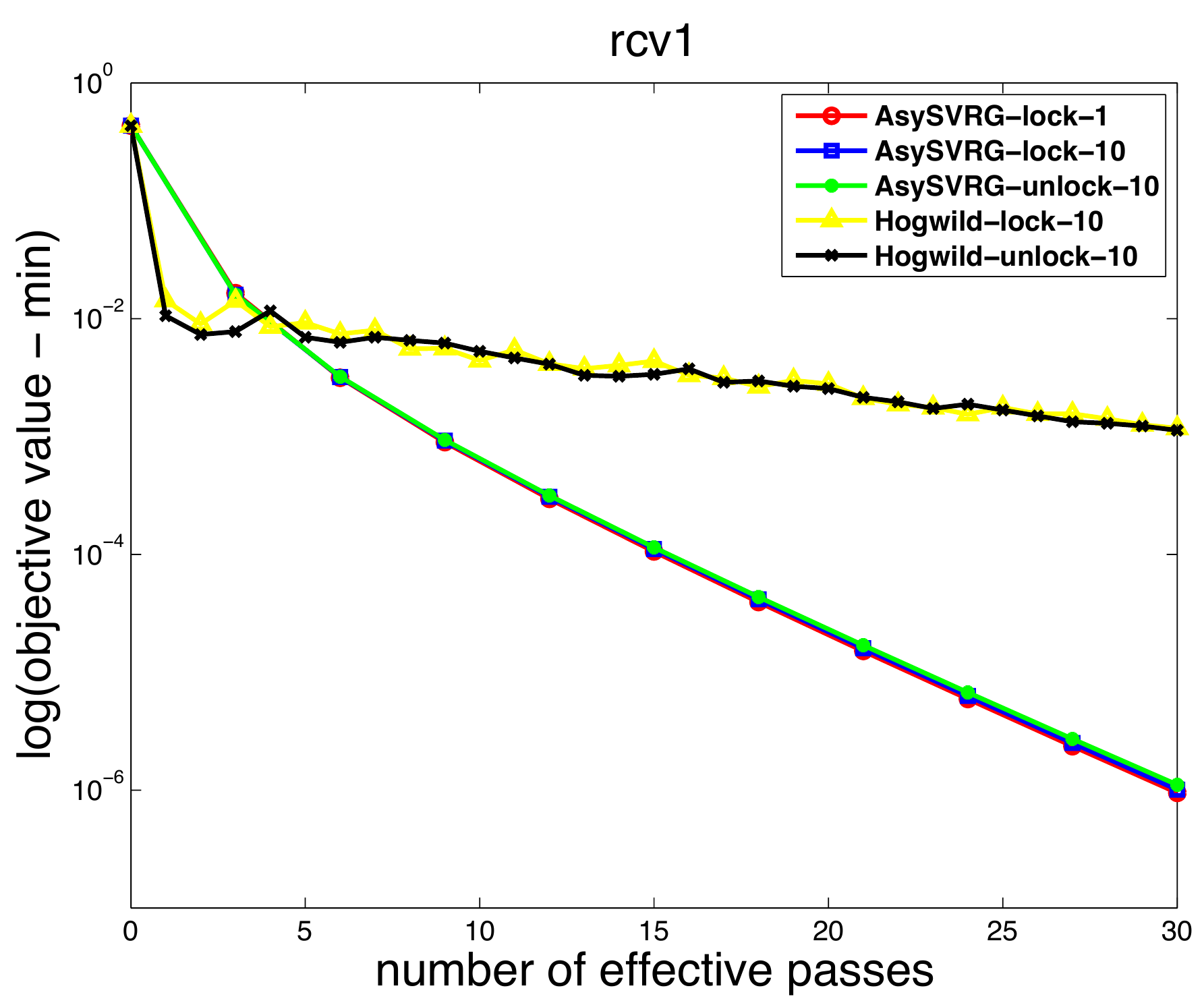}}\hspace{0.5cm}
\subfigure[realsim]{\includegraphics[width=1.5in]{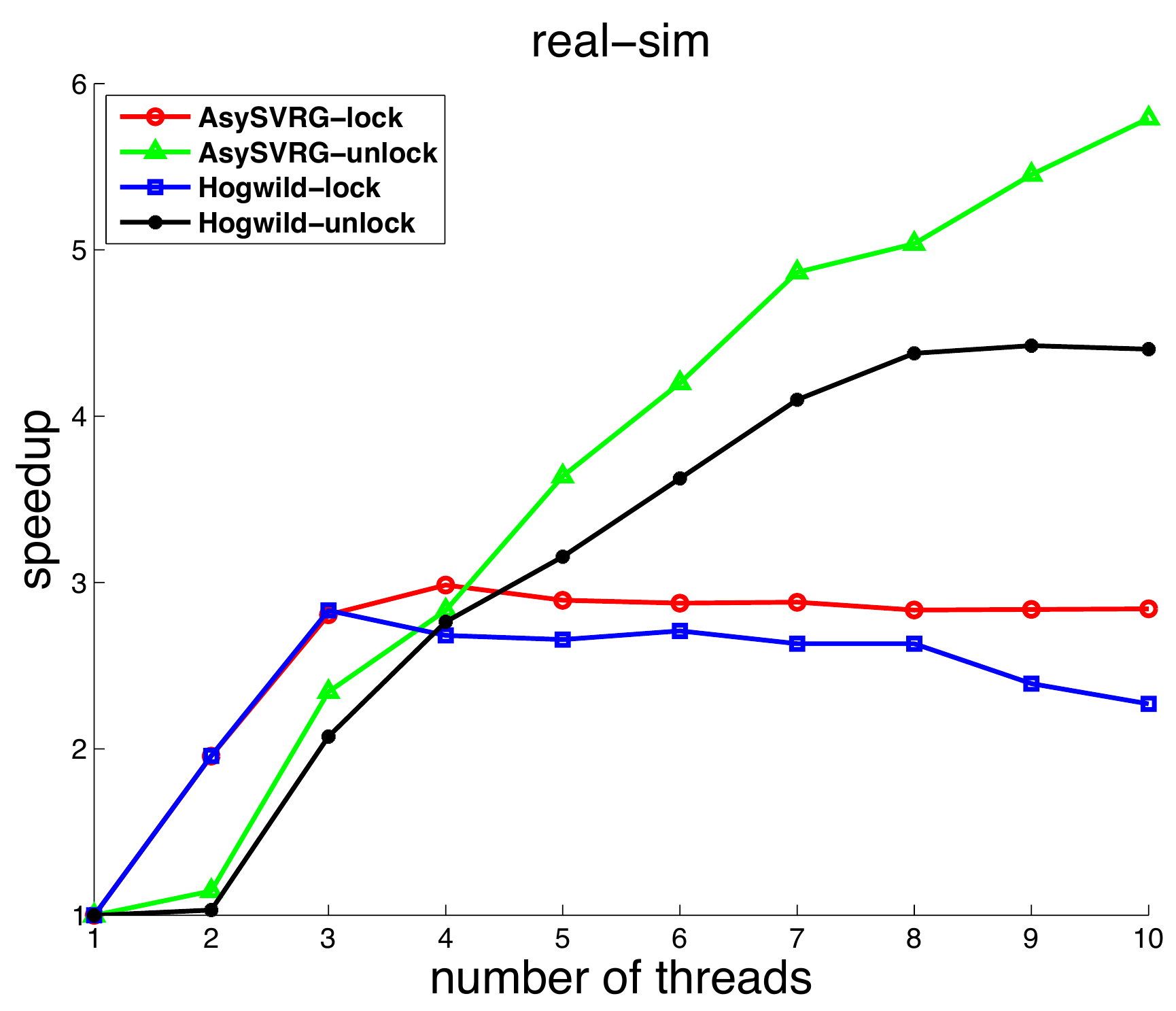}}\hspace{0.5cm}
\subfigure[realsim]{\includegraphics[width=1.5in]{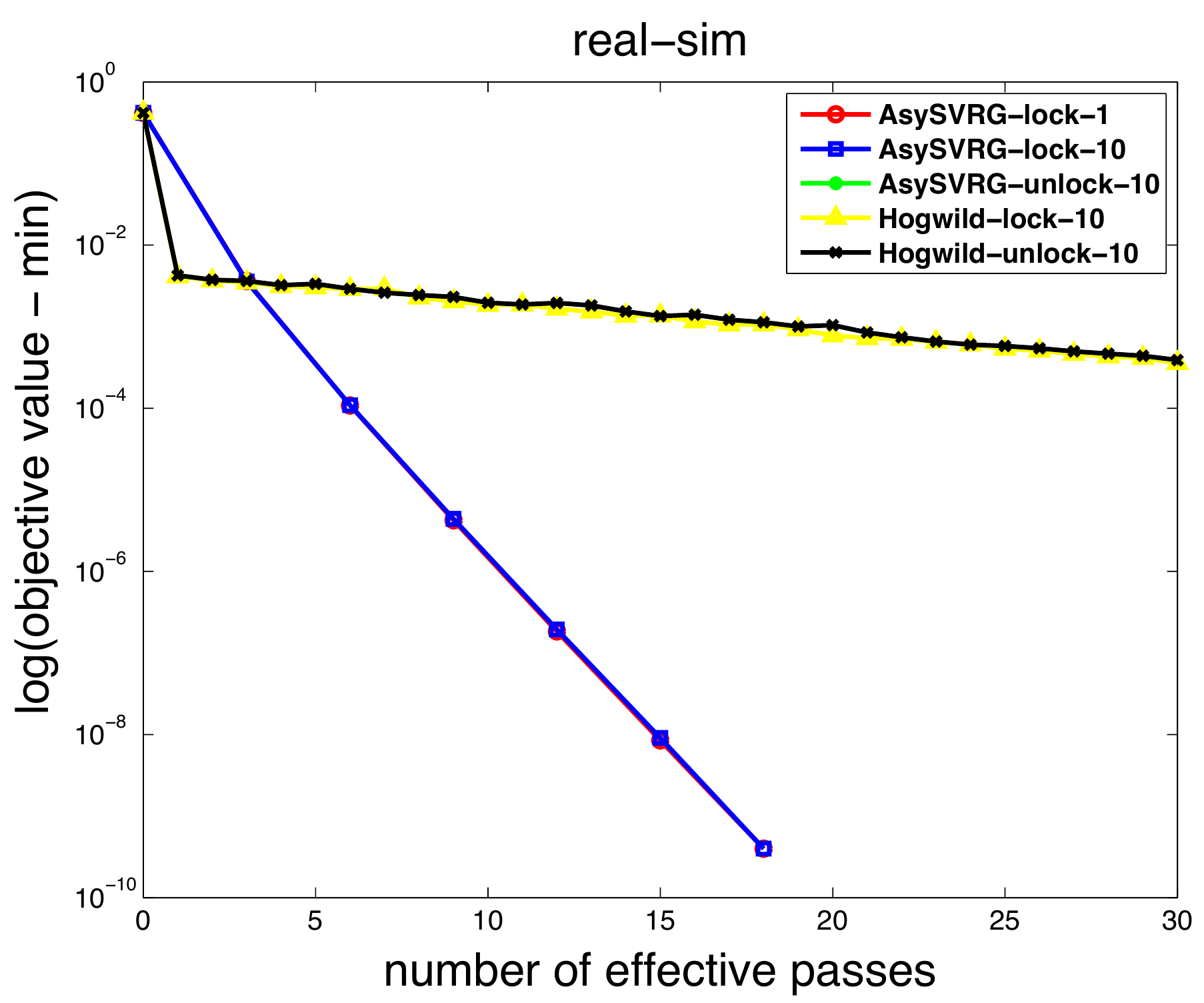}}\hspace{0.5cm}
\subfigure[news20]{\includegraphics[width=1.5in]{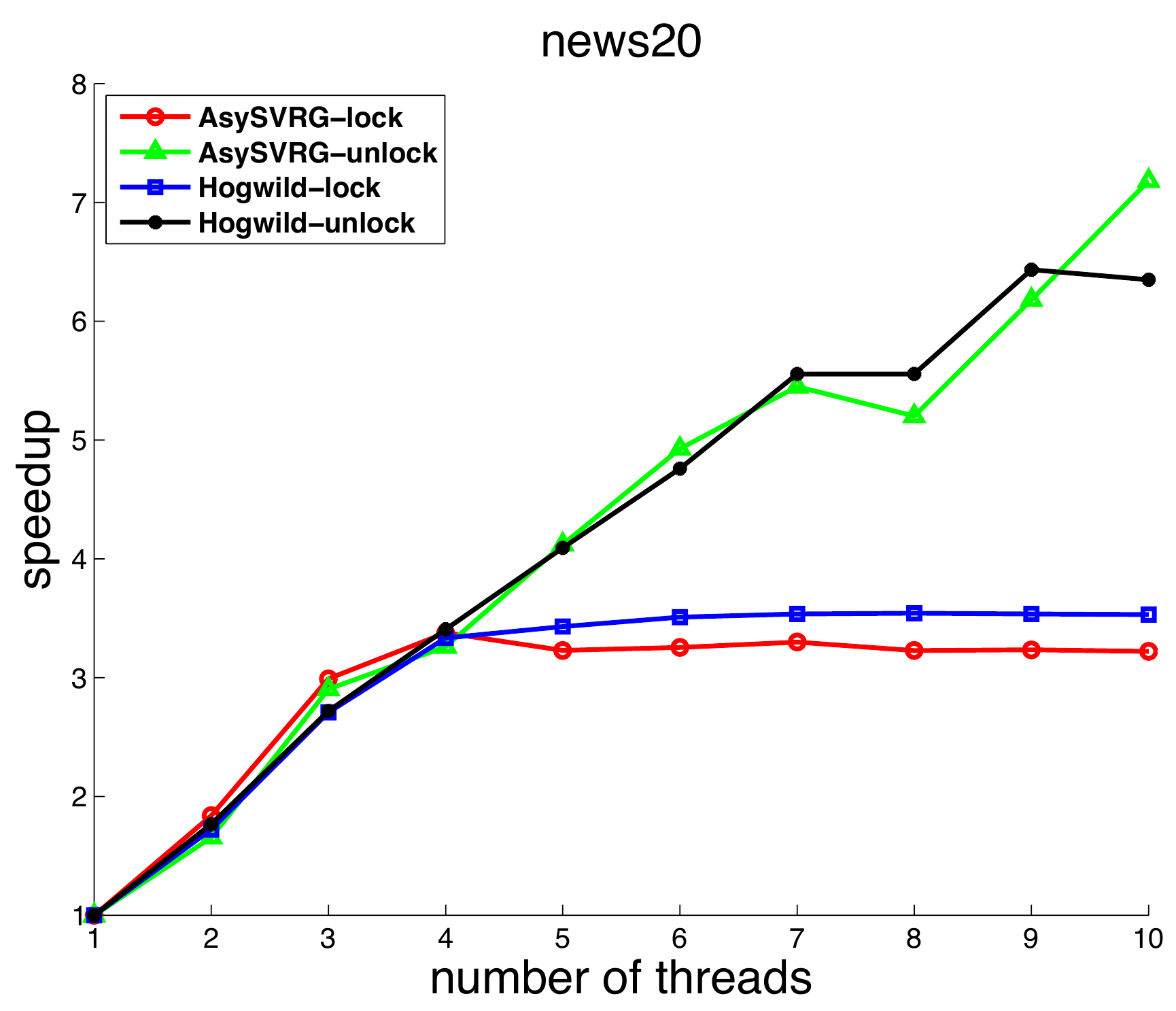}}\hspace{0.5cm}
\subfigure[news20]{\includegraphics[width=1.5in]{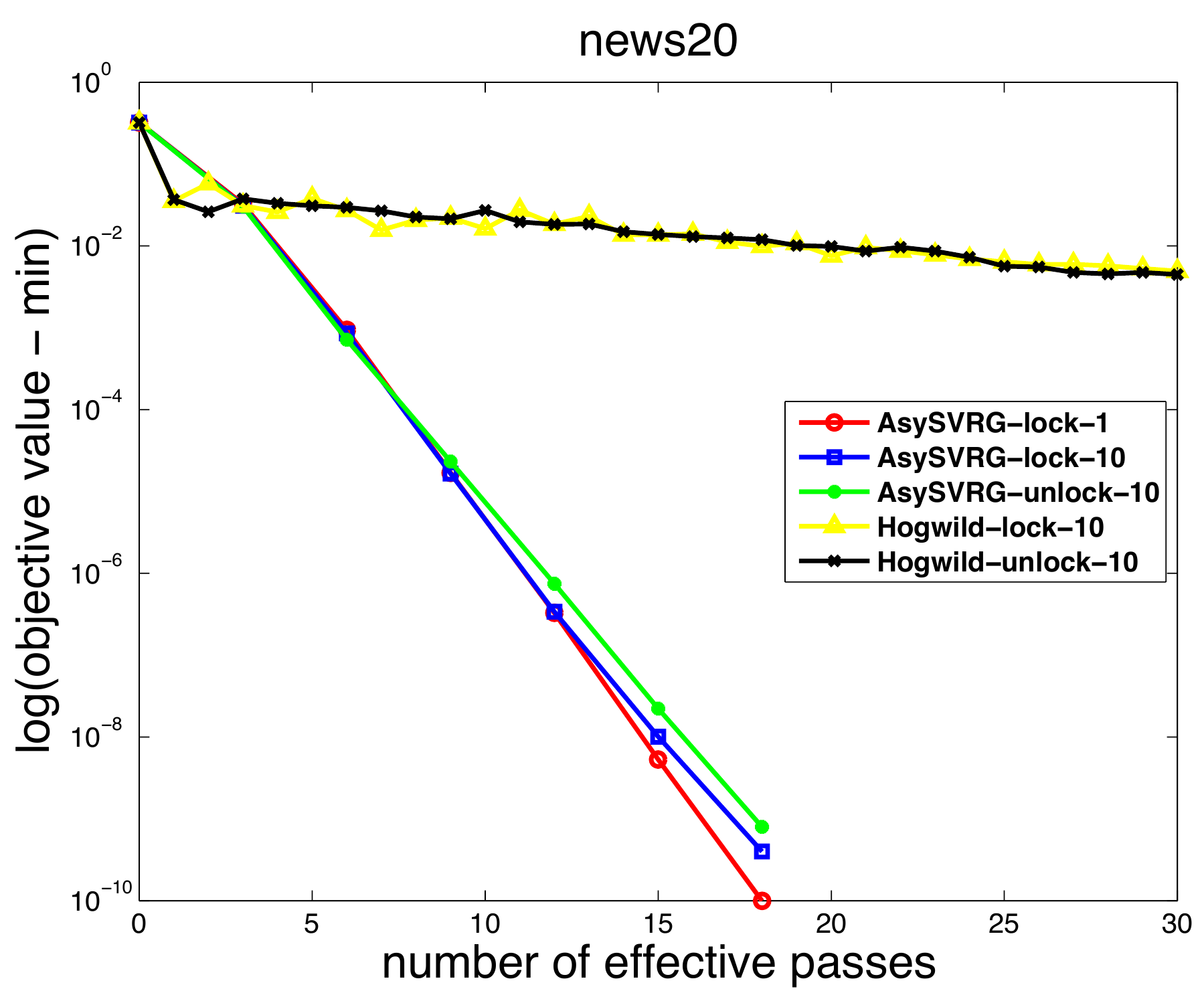}}\hspace{0.5cm}
\end{center}
\vspace{-0.5cm}
\caption{Experimental results. Left: speedup; Right: convergence rate. Please note that in (b), (d) and (f), some curves are overlapped.}
\label{fig}
\end{figure*}

\section{Conclusion}
In this paper, we have proposed a novel asynchronous parallel SGD method, called AsySVRG, for multicore systems. Both theoretical and empirical results show that AsySVRG can outperform other state-of-the-art methods.

\bibliographystyle{plain}
\bibliography{ref}

\newpage
\appendix

\section{Notations for Proof}
For the proof of Theorem \ref{theorem1}.

\begin{proof}
According to (\ref{update rule}), we obtain that
\begin{align}\label{eq}
\EB \left\| \u_{m+1} - \w_* \right\|^2 &= \EB\left\| \u_m - \w_* \right\|^2 - 2\eta \EB \v_m^T (\u_m - \w_*) + \eta^2 \EB \left\| \v_m \right\|^2 \nonumber \\
&=\EB \left\| \u_m - \w_* \right\|^2 - 2\eta \EB \nabla f(\u_{k(m)})^T (\u_m - \w_*) + \eta^2 \EB \left\| \v_m \right\|^2.
\end{align}

For the old gradient, we have
\begin{align}\label{old gradient}
\nabla f(\u_{k(m)})^T (\u_m - \w_*) =& \nabla f(\u_{k(m)})^T (\u_{k(m)} - \w_*) + \nabla f(\u_{k(m)})^T (\u_m - \u_{k(m)}).
\end{align}

Since $f(\w)$ is $L$-smooth, we have
\begin{align}
&\nabla f(\u_{k(m)})^T (\u_{k(m)} - \w_*) \geq f(\u_{k(m)}) - f(\w_*), \nonumber \\
&\nabla f(\u_{k(m)})^T (\u_m - \u_{k(m)}) \geq f(\u_m) - f(\u_{k(m)}) - \frac{L}{2} \left\| \u_m - \u_{k(m)} \right\|^2. \nonumber
\end{align}

Substituting the above inequalities into (\ref{eq}), we obtain
\begin{align}
&\EB \left\| \u_{m+1} - \w_* \right\|^2 + 2\eta \EB(f(\u_m) - f(\w_*)) \nonumber \\
\leq& \left\| \u_m - \w_* \right\|^2 + \eta^2 \EB \left\| \v_m \right\|^2 + \eta L(\left\| \u_m - \u_{k(m)} \right\|^2).
\end{align}
Since
\begin{align}
\left\| \u_m - \u_{k(m)} \right\|^2 \leq 2\sum_{l=k(m)}^{m-1} \left\| \u_{l+1} - \u_l \right\|^2 = 2\eta^2\sum_{l=k(m)}^{m-1} \left\| \v_l \right\|^2 .\nonumber
\end{align}
Taking expectation and using Lemma \ref{lemma1}, we obtain
\begin{align}
& \EB \left\| \u_{m+1} - \w_* \right\|^2 + 2\eta \EB(f(\u_m) - f(\w_*)) \nonumber \\
\leq& \EB\left\| \u_m - \w_* \right\|^2 + \eta^2 \EB \left\| \v_m \right\|^2 + 2\eta^3 L \sum_{l=k(m)}^{m-1} \EB \left\| \v_l \right\|^2 \nonumber \\
\leq& \EB\left\| \u_m - \w_* \right\|^2 + \eta^2 \sum_{l=k(m)}^m \EB\left\| \v_l \right\|^2 \nonumber \\
=& \EB\left\| \u_m - \w_* \right\|^2 + \eta^2 \sum_{l=k(m)}^m \EB q_{k(l)} \nonumber \\
\leq& \EB\left\| \u_m - \w_* \right\|^2 + \rho^\tau \eta^2 \sum_{l=k(m)}^m \EB q_l \nonumber \\
\leq& \EB\left\| \u_m - \w_* \right\|^2 + (\tau + 1) \rho^{2\tau} \eta^2 \EB q_m \nonumber \\
\leq& \EB\left\| \u_m - \w_* \right\|^2 + 4 (\tau + 1) \rho^{2\tau} \eta^2 L\EB(f(\u_m) - f(\w_*) + f(\u_0) - f(\w_*)). \nonumber
\end{align}
The second inequality uses $2L\eta \leq 1$. In the first equality, we use the fact that $\EB\left\| \v_l \right\|^2 = \EB (\EB_{i_m}(\left\| \v_l \right\|^2)) = \EB q_{k(l)}$. The last inequality uses the inequality
$$q_m \leq 4L(f(\u_m) - f(\w_*) + f(\u_0) - f(\w_*)),$$
which has been proved in \cite{DBLP:conf/nips/Johnson013}.
Then summing up from $m=0$ to $\tilde{M}$, and taking $\w_{t+1} = \frac{1}{\tilde{M}} \sum_{m=0}^{\tilde{M}-1} \u_m$ or randomly choosing a $\u_m$ to be $\w_{t+1}$, we can get
\begin{align}
&2\tilde{M} \eta (1 - 2 (\tau + 1) \rho^{2\tau} \eta L)\EB(f(\w_{t+1}) - f(\w_*)) \\
\leq& \EB\left\| \w_t - \w_* \right\|^2 + 4M(\tau + 1) \rho^{2\tau} \eta^2 L \EB(f(\w_t) - f(\w_*)) \nonumber \\
\leq& (\frac{2}{\mu} +  4\tilde{M}(\tau + 1) \rho^{2\tau} \eta^2 L)\EB(f(\w_t) - f(\w_*)). \nonumber
\end{align}
Then, we have
\begin{align}
\EB(f(\w_{t+1}) - f(\w_*)) \leq (\frac{1}{\mu \tilde{M} \eta (1 - 2 (\tau + 1) \rho^{2\tau} \eta L)} +  \frac{2(\tau + 1) \rho^{2\tau} \eta L}{1 - 2 (\tau + 1) \rho^{2\tau} \eta L} )\EB(f(\w_t) - f(\w_*)). \nonumber
\end{align}
Of course, we need $1 - 2 (\tau + 1) \rho^{2\tau} \eta L>0$.
\end{proof}

For the proof of Theorem \ref{theorem:inconsistent}.

\begin{proof}
According to (\ref{update:inconsistent}), we have
\begin{align}
\EB \left\| \u_{m+1} - \w_* \right\|^2 &= \EB\left\| \u_m - \w_* \right\|^2 - 2\eta \EB \hat{\v}_m^T (\u_m - \w_*) + \eta^2 \EB \left\| \hat{\v}_m \right\|^2 \nonumber \\
&=\EB \left\| \u_m - \w_* \right\|^2 - 2\eta \EB \nabla f(\hat{\u}_m)^T (\u_m - \w_*) + \eta^2 \EB \left\| \hat{\v}_m \right\|^2.
\end{align}
Similar to the analysis of (\ref{old gradient}) in Theorem \ref{theorem1}, we can get
\begin{align}
    &\EB \left\| \u_{m+1} - \w_* \right\|^2 + 2\eta \EB(f(\u_m) - f(\w_*)) \nonumber \\
\leq& \left\| \u_m - \w_* \right\|^2 + \eta^2 \EB \left\| \hat{\v}_m \right\|^2 + \eta L\left\| \u_m - \hat{\u}_m \right\|^2 \nonumber \\
\leq& \left\| \u_m - \w_* \right\|^2 + \eta^2 \EB \left\| \hat{\v}_m \right\|^2 + 4L\eta^3 \sum_{l=a(m)}^{m-1} \EB \left\| \hat{\v}_l \right\|^2 \nonumber \\
\leq& \left\| \u_m - \w_* \right\|^2 + \eta^2 \EB \left\| \hat{\v}_m \right\|^2 + 4\tau\rho^\tau L \eta^3 \EB \left\| \hat{\v}_m \right\|^2 \nonumber \\
\leq& \left\| \u_m - \w_* \right\|^2 + \frac{\eta^2 + 4\tau\rho^\tau L \eta^3}{1 - \frac{1}{r} - 4r\tau\rho^\tau\eta^2L^2} \EB\left\| \v_m \right\|^2 \nonumber \\
\leq& \left\| \u_m - \w_* \right\|^2 + \frac{4L\eta^2 + 16\tau\rho^\tau L^2 \eta^3}{1 - \frac{1}{r} - 4r\tau\rho^\tau\eta^2L^2}(f(\u_m) - f(\w_*) + f(\u_0) - f(\w_*)) .\nonumber
\end{align}
The second inequality is the same as the analysis in (\ref{EBhatvm}). The third inequality uses Lemma \ref{estimate:hatq}. The fourth inequality uses Lemma \ref{estimate:hatvm}.

For convenience, we use $c_2 = \frac{4L\eta^2 + 16\tau\rho^\tau L^2 \eta^3}{1 - \frac{1}{r} - 4r\tau\rho^\tau\eta^2L^2}$ and sum the above inequality from $m=0$ to $\tilde{M}-1$, and
 take $\w_{t+1} = \frac{1}{\tilde{M}}\sum_{m=0}^{\tilde{M}-1} \u_{m}$. Then, we obtain
\begin{align}
(2\eta - c_2)M \EB(f(\w_{t+1}) - f(\w_*)) \leq (\frac{2}{\mu} + \tilde{M}c_2) \EB(f(\w_t) - f(\w_*)), \nonumber
\end{align}
which means that
\begin{align}
\EB(f(\w_{t+1}) - f(\w_*)) \leq (\frac{2}{\mu \tilde{M}(2\eta - c_2)} + \frac{c_2}{2\eta - c_2}) \EB(f(\w_t) - f(\w_*)) .\nonumber
\end{align}
\end{proof}

\end{document}